\documentclass{article}

\usepackage{microtype}
\usepackage{graphicx}
\usepackage{subcaption}
\usepackage{booktabs} 

\usepackage{amssymb}
\usepackage{amsmath}
\usepackage{amsthm}
\usepackage{enumitem}
\usepackage{mathtools}

\usepackage{xcite}

\usepackage{xr-hyper}
\makeatletter
\newcommand*{\addFileDependency}[1]{
  \typeout{(#1)}
  \@addtofilelist{#1}
  \IfFileExists{#1}{}{\typeout{No file #1.}}
}
\makeatother

\DeclarePairedDelimiter\cbracket{\{}{\}}
\DeclarePairedDelimiter\norm{\|}{\|}
\DeclarePairedDelimiter\bracket{[}{]}
\DeclarePairedDelimiter\paran{(}{)}

\DeclareMathOperator{\EE}{\mathbb{E}}

\DeclareMathOperator{\Unif}{Unif}
\newcommand{\Qopt}[1]{Q_{\vec{p}^*(#1)}}
\renewcommand{\vec}[1]{\mathbf{#1}}

\newtheorem{lemma}{Lemma}
\newtheorem{definition}{Definition}
\newtheorem{thm}{Theorem}
\DeclareMathOperator*{\argmax}{arg\,max}
\DeclareMathOperator*{\argmin}{arg\,min}

\setlist[itemize]{noitemsep}
\newcommand{\rootpath}{.}
\newcommand{\figurepath}{\rootpath/figures}

\usepackage{hyperref}

\usepackage[accepted]{icml2021_draft}

\icmltitlerunning{Out-of-Distribution Robustness in Deep Learning Compression}

\begin{document}

\twocolumn[
\icmltitle{Out-of-Distribution Robustness in Deep Learning Compression}




\begin{icmlauthorlist}
\icmlauthor{Eric Lei}{ese}
\icmlauthor{Hamed Hassani}{ese}
\icmlauthor{Shirin Saeedi Bidokhti}{ese}

\end{icmlauthorlist}

\icmlaffiliation{ese}{Department of Electrical and Systems Engineering, University of Pennsylvania, Philadelphia, PA, USA}

\icmlcorrespondingauthor{Eric Lei}{elei@seas.upenn.edu}

\icmlkeywords{Compression, Machine Learning, Robustness}

\vskip 0.3in
]



\printAffiliationsAndNotice{}  

\begin{abstract}
In recent years, deep neural network (DNN) compression systems have proved to be highly effective for designing source codes for many natural sources. However, like many other machine learning  systems, these compressors suffer from vulnerabilities to distribution shifts as well as out-of-distribution (OOD) data, which reduces their real-world applications. In this paper, we initiate the study of OOD robust compression. Considering  robustness to two types of ambiguity sets (Wasserstein balls and group shifts), we propose algorithmic and architectural frameworks built on two principled methods: one that trains DNN compressors using distributionally-robust optimization (DRO), and the other which uses a structured latent code. Our results demonstrate that both methods enforce robustness compared to a standard DNN compressor, and that using a structured code can be superior to the DRO compressor. We observe tradeoffs between robustness and distortion and corroborate these findings theoretically for a specific class of sources.  
\end{abstract}

\section{Introduction}
\label{intro}
Deep neural network-based compressors have achieved great success in lossy source coding, oftentimes outperforming traditional compression schemes on real-world data in terms of minimizing distortion and producing visually pleasing reconstructions at reasonable complexity \cite{Balle2017, Theis2017a, SoftToHardVQ, NTC, WagnerDNN}. These methods typically use an autoencoder architecture with quantization of the latent variables, which is trained over samples drawn from the source (i.e. data distribution) $X \sim P_D$ to jointly minimize rate and the distortion between $X$ and its reconstruction $\hat{X}$. The performance of such methods can be attributed to the ability of DNNs to approximate arbitrary functions and generalize to $P_D$. 

	
While these compressors have been shown to perform well on test data drawn from the same distribution as the training data (\emph{in-distribution} data), they are highly susceptible to simple modifications of the data they are trained on.
In Fig.~\ref{fig:OOD_example}, we demonstrate how slight modifications such as rotations, noise, and small perturbations prior to compression can highly distort the decompressed image compared to images that follow the base distribution. Comparatively, traditional compressors like JPEG do not suffer from such abrupt performance degradations, and perform more consistently for natural images. The lack of robustness in DNN compression, along with learned compressors' dependence on training data, limits its application in the real-world and as a replacement for traditional compressors. Hence, it is imperative to develop DNN compressors that provide resilience to perturbed or \emph{out-of-distribution} (OOD) data. 

\begin{figure}[t]
    \centering
    \begin{subfigure}{.35\textwidth}
        \centering
        \includegraphics[width=1\linewidth]{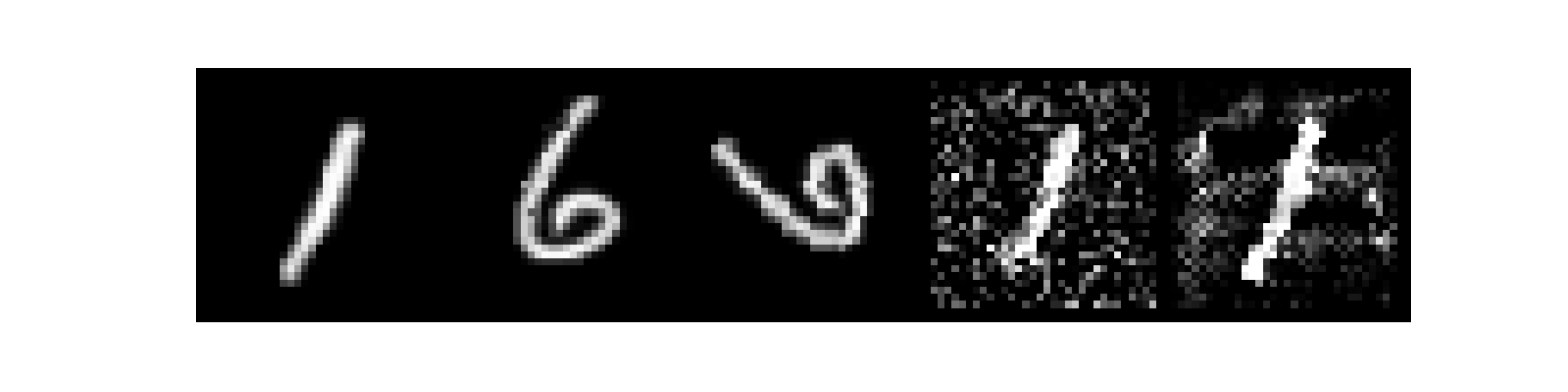}
        \vspace{-2em}
        \caption{Prior to compression.}
    \end{subfigure}%
    \\
    \begin{subfigure}{.35\textwidth}
        \centering
        \includegraphics[width=1\linewidth]{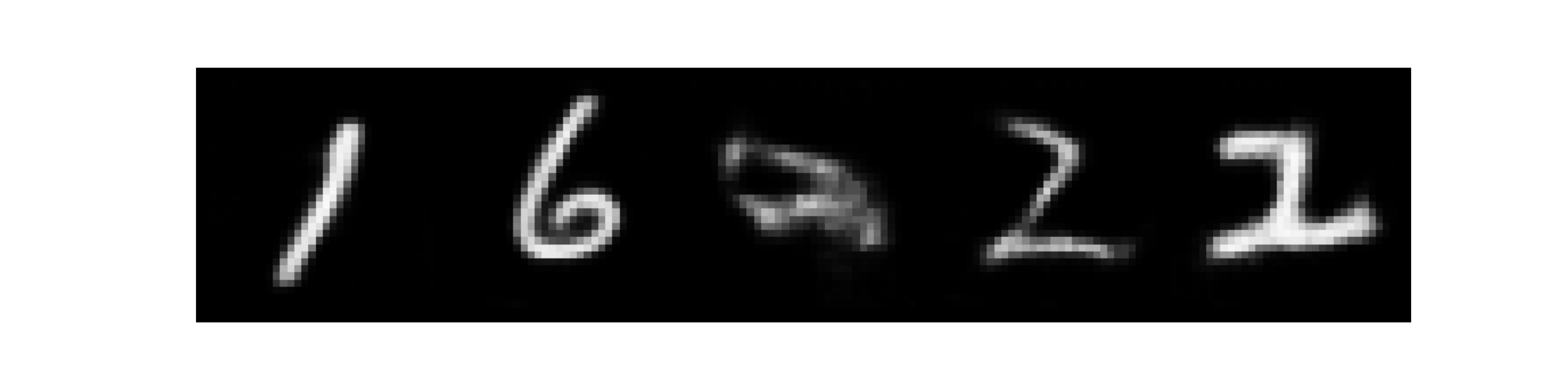}
        \vspace{-2em}
        \caption{Decompressed.}
    \end{subfigure}%
    \caption{Effect of DNN compressors on slightly modified versions of a MNIST ``1'' and ``6''. Left to right: original, original, rotated, additive white Gaussian noise, adversarial noise.}
    \label{fig:OOD_example}
\end{figure}

In the information theory literature,  settings such as noisy source coding and universal lossy compression \cite{NoisySourceCoding, SingleShotRedundancy, ChouEffros}, which may provide resilience to OOD sources, have been previously investigated. However, while the difficulty of compressing real-world sources and high computational requirements of optimal compression schemes (e.g. vector quantization) have been mitigated by DNN models, these universal settings pose an even harder problem, and remain highly intractable in practice. Can DNNs also help bridge the gap between these settings and practical schemes?

Further questions remain. In designing compressors that provide robustness to OOD inputs, is there a price to pay in terms of their performance on in-distribution data? If so, are these tradeoffs \textit{fundamental}, in that no amount of data or computational power can help remove such tradeoffs? Lastly, do we need to rethink how DNN-based compressors are designed in order to help make them more robust, rather than learning DNN-based codes in an ad-hoc manner? 

In what follows, we provide a novel framework for designing robust compressors based on distributionally-robust optimization (DRO) and inspired by learning under distribution shifts \cite{biggio2013evasion,szegedy,carlini2017,madry2017towards, hendrycks2019natural, Duchi2018LearningMW, Robey2020ModelBasedRD, Taori2020MeasuringRT} which is a parallel area of research. We provide two solutions using DNNs. The first uses an end-to-end architecture that solves a DRO min-max objective function. The second method, inspired by successive refinement from information theory \cite{SuccessiveRefinement}, imposes structure within the code (in the latent space), where some portion of the bits are used to encode the source $P_D$, and additional bits help encode perturbations from $P_D$.

\vspace{-1.em}
\paragraph{Contributions.} Our contributions are as follows. \vspace{-1.1em}
\begin{itemize} 
    \item We formulate a new problem for designing robust compressors in which we enforce robustness by minimizing worst-case distortions over a set of distributions. We first focus on a specific set of distributions: the Wasserstein ball centered at $P_D$.
    \item We provide an algorithm to implement such a compressor in an end-to-end fashion by leveraging advances made in distributionally-robust optimization, and a structured architecture based on successive refinement, where the learned code can progressively encode more information for distributions that are further OOD. 
    \item We empirically show that both methods are effective in providing robustness to OOD inputs, and that enforcing structure in the model architecture can actually help it outperform simply optimizing the DRO objective.
    \item We additionally evaluate our framework on distributions induced by group shifts of $X \sim P_D$. We show that adding structure in the model still results in enforcing more robustness. 
    \item We observe that there is a tradeoff between minimizing distortion on OOD inputs and those that are in-distribution, and that the tradeoff gets worse as more robustness is desired. We corroborate these empirical findings with an analytical tradeoff for single-shot codes when compressing a set of uniform sources.  
\end{itemize}


\section{Problem Formulation}
In all the data compression settings in this paper, we only consider single-shot (non-asymptotic) scenarios since DNN compressors typically operate only on single realizations.
\subsection{Preliminaries: Standard Compressor}
 Let $\mathcal{W}\subset \mathbb{R}$ be a discrete set, and let $(f,g)$ be a fixed-rate single-shot code of rate $m \log |\mathcal{W}|$ where $f:\mathbb{R}^n \rightarrow \mathcal{W}^m$ is an encoder and $g:\mathcal{W}^m \rightarrow \mathbb{R}^n$ is a decoder.  In  standard compression, we want to compress a source $X \sim P_D \in \mathcal{P}(\mathbb{R}^n)$ by finding the code that minimizes the expected distortion. Standard DNN compressors simply parametrize $f$, $g$ as DNNs $f_\theta$, $g_\phi$ with parameters $\theta$, $\phi$ respectively, where $f_\theta$ also quantizes each dimension of its output layer independently to values in $\mathcal{W}$ using a nearest-neighbor quantizer.

Since we only consider fixed-rate codes, we do not try to minimize the expected length of the codes, which is typically done by minimizing the entropy of the quantizer output. Hence, the objective is solely to minimize the expected distortion using standard DNN training methods:
\begin{equation}
    \min_{\theta,\phi} \EE_{P_D}[d(X, g_\phi(f_\theta(X)))]
    \label{eq:standard}
\end{equation}
where $d:\mathbb{R}^n \times \mathbb{R}^n \rightarrow \mathbb{R}$ is a distortion measure. 

The only non-trivial aspect of training is when performing backpropagation on the encoder $f_\theta$, since the quantizer is not differentiable. We use the method described in \cite{SoftToHardVQ}, where the quantizer is used during forward passes, but not during backpropagation, where the gradient of a soft quantizer (softmax function) is used instead.

\subsection{Out-of-Distribution-Robust Compression}
We now introduce the robust single-shot compression problem. In learning-based models, training is performed on data that comes from a data distribution $P_D$. In real-world, however, test data does not always come from $P_D$, and instead, it may come from a shifted version $P'_D$ which is often unknown.   For example, if a model is trained on image data, images during test time may go under different shifts such as rotation, translation, blurriness, contrast, or in natural images there may be different weather conditions or backgrounds, which the model may not have seen during training. As shown in fig.~\ref{fig:OOD_example}, this problem can be catastrophic. 

One way to approach this problem is to assume that the distributions of data seen at test time belong to a family of distributions $\mathcal{P}$ and solve the minimax distortion problem:
\begin{equation}
    \min_{\theta, \phi} \sup_{P \in \mathcal{P}} \mathbb{E}_{P} [d(X, g_\phi(f_\theta(X)))].
    \label{eq:problem_formulation}
\end{equation}

Here, we can establish connections to previous results in information theory. Many of the those results consider codes that minimize the worst-case distortion-redundancy \cite{NoisySourceCoding, UniversalAsymptotic}; that is, the difference between the distortion incurred using a universal code and that of the optimal code. In asymptotic settings, it is well-known that universal codes exist that are optimal for all sources in the considered set $\mathcal{P}$, and single-shot results have been shown more recently in \cite{SingleShotRedundancy}. 

In what follows, we study two specific cases of $\mathcal{P}$. The first one chooses $\mathcal{P}$ in \eqref{eq:problem_formulation} to be a Wasserstein ball of radius $\rho$ around $P_D$, i.e. 
\begin{equation}
    \min_{\theta, \phi} \sup_{P: W_c(P, P_D) < \rho} \EE_{P}\bracket*{d(X, g_\phi(f_\theta(X)))} 
    \label{eq:robust}
\end{equation}
where the $p$-Wasserstein distance between $\mu$ and $\nu$ is $W_c(\mu,\nu):=\left(\inf_{\pi\in\Pi(\mu,\nu)} \mathbb{E}_{X,X'\sim\pi} [c(X,X')^p]\right)^{1/p}$ with $c:\mathbb{R}^n \times \mathbb{R}^n \rightarrow [0,\infty)$ the transportation cost. (Small) translations, blurring, noise and other types of distribution shifts can be captured through this model. 
Being robust to all distributions in $W_c(P, P_D)$ is, hence, a reasonable design requirement for any compressor whereas the generic set  $\mathcal{P}$ may be too broad and not  as practically relevant. In addition, the formulation in \eqref{eq:robust}  allows us to control the level of robustness with $\rho$ while also ensuring the set of distributions considered is rich enough, as opposed to $f$-divergence balls which are more limited in containing relevant distributions \cite{Gao2016DistributionallyRS}. In theory and design, using the Wasserstein distance allows for tractable implementations during optimization of \eqref{eq:robust}, as will be described in the next section. 

The second case of $\mathcal{P}$ pertains to group shifts on $X \sim P_D$. Specifically, we consider the set of distributions induced group actions $c \in \mathcal{G}$ on $X$, i.e. $\mathcal{P} = \{P_{c\cdot X}: c \in \mathcal{G} \}$ with the objective
\begin{equation}
    \min_{\theta, \phi} \sup_{c \in \mathcal{G}} \EE_{P_D} [d(c\cdot X, g_\phi(f_\theta(c \cdot X)))]
\end{equation}
This choice of $\mathcal{P}$ can capture more naturally occuring distribution shifts, such as all spatial transformations (shift, warp, rotate, etc.). For organizational purposes, we will further discuss this setting in Section~\ref{sec:group}.

We next propose two methods to solve \eqref{eq:robust}, and demonstrate tradeoffs that arise compared to solving \eqref{eq:standard}. 

\section{Robust Neural-Network Compression}

\subsection{Enforcing Distributional Robustness}
\begin{algorithm}[tb]
   \caption{End-to-End DRO Compressor}
   \label{alg:DRO}
    \begin{algorithmic}
       \WHILE{not converged}
       \STATE Sample a batch $x \sim P_D$, and let $x_0' = x$
       \FOR{$k=0, \dots, K-1$}
       \STATE \hspace{-1em} $x'_{k+1} \leftarrow  x'_k +  \frac{1}{\sqrt{k}}\nabla_{x'} \paran*{d(x', g_\phi(f_\theta(x'))) - \gamma c(x', x)}$
       \ENDFOR
       \STATE{$\theta_{t+1} \leftarrow \theta_t - \eta_t \nabla_{\theta} d(x'_K, g_{\phi_t}(f_{\theta_t}(x'_K))) $}
       \STATE{$\phi_{t+1} \leftarrow \phi_t - \eta_t \nabla_{\phi} d(x'_K, g_{\phi_t}(f_{\theta_t}(x'_K))) $}
       \ENDWHILE
    \end{algorithmic}
\end{algorithm}

In order to solve \eqref{eq:robust}, we leverage duality results in \cite{Gao2016DistributionallyRS, CertifiedDRO, Blanchet2016QuantifyingDM, Esfahani2018DatadrivenDR} which show that for data-driven DRO problems (i.e. when $P_D$ is a sum of point masses), \eqref{eq:robust} can be approximated to any accuracy by performing adversarial training on the data samples themselves. Concretely, fix $\gamma > 0$ and suppose $P_D = \frac{1}{N}\sum_{i=1}^N \delta_{x_i}$. Let $x_i^* =  \argmax_{x'} \cbracket*{d(x', g_\phi(f_\theta(x'))) - \gamma c(x', x_i)}$. Then, the worst-case distortion-achieving distribution is $P^* = \frac{1}{N}\sum_{i=1}^N \delta_{x_i^*}$ over the ball of radius $\hat{\rho} = \mathbb{E}_{P_D}[c(x_i^*, x_i)]$.

Hence, we focus on the Lagrangian \cite{CertifiedDRO}:
\begin{align}
    & \min_{\theta,\phi} \sup_P \cbracket*{\EE_P[d(X,g_\phi(f_\theta(X)))] - \gamma W_c(P,P_D)} \nonumber \\ &= \min_{\theta,\phi} \EE_{P_D} \bracket*{\sup_{x'} \cbracket*{d(x', g_\phi(f_\theta(x'))) - \gamma c(x', x)}}
    \label{eq:lagrangian}
\end{align}
which now has a tractable alternating min-max procedure, described in Alg.~\ref{alg:DRO}. We call this the end-to-end DRO compressor since it uses nothing other than the same model architecture as the standard compressor and is similarly trained end-to-end, except it solves the objective in \eqref{eq:lagrangian}.

\subsection{Structured Compressors}
Here, we introduce an idea inspired by the concept of successive refinement from information theory, in which a source can be first coarsely described using a small number of bits, followed by additional bits that iteratively provide a more accurate description. Successive refinement through DNNs has recently been applied in  joint source-channel coding  \cite{Kurka2019SuccessiveRefinement}. We, however, investigate such encoding structures for DNN universality and robustness. 

In the context of the robust problem \eqref{eq:robust}, the idea is to first describe a smaller set of sources using some number of bits, followed by larger balls around $P_D$ using more bits. Imposing such a structure on the encoder, one might be able to more easily optimize for a robust compressor. We will consider a simplified case where we have two levels of refinement; the first stage of the compressor uses $R_1$ bits to describe the center of the ball $P_D$, and the second stage uses $R_1+R_2$ bits to describe the ball $\{P:W_c(P,P_D)<\rho\}$.


To implement this, we use a single encoder DNN $f_1$ and two decoder DNNs $g_1$ and $g_2$, shown in Fig.~\ref{fig:structured}. This is similar to the multiple-decoders architecture described in \cite{Kurka2019SuccessiveRefinement}. 
\begin{figure}[t]
    \centering
    \includegraphics[width=0.4\textwidth]{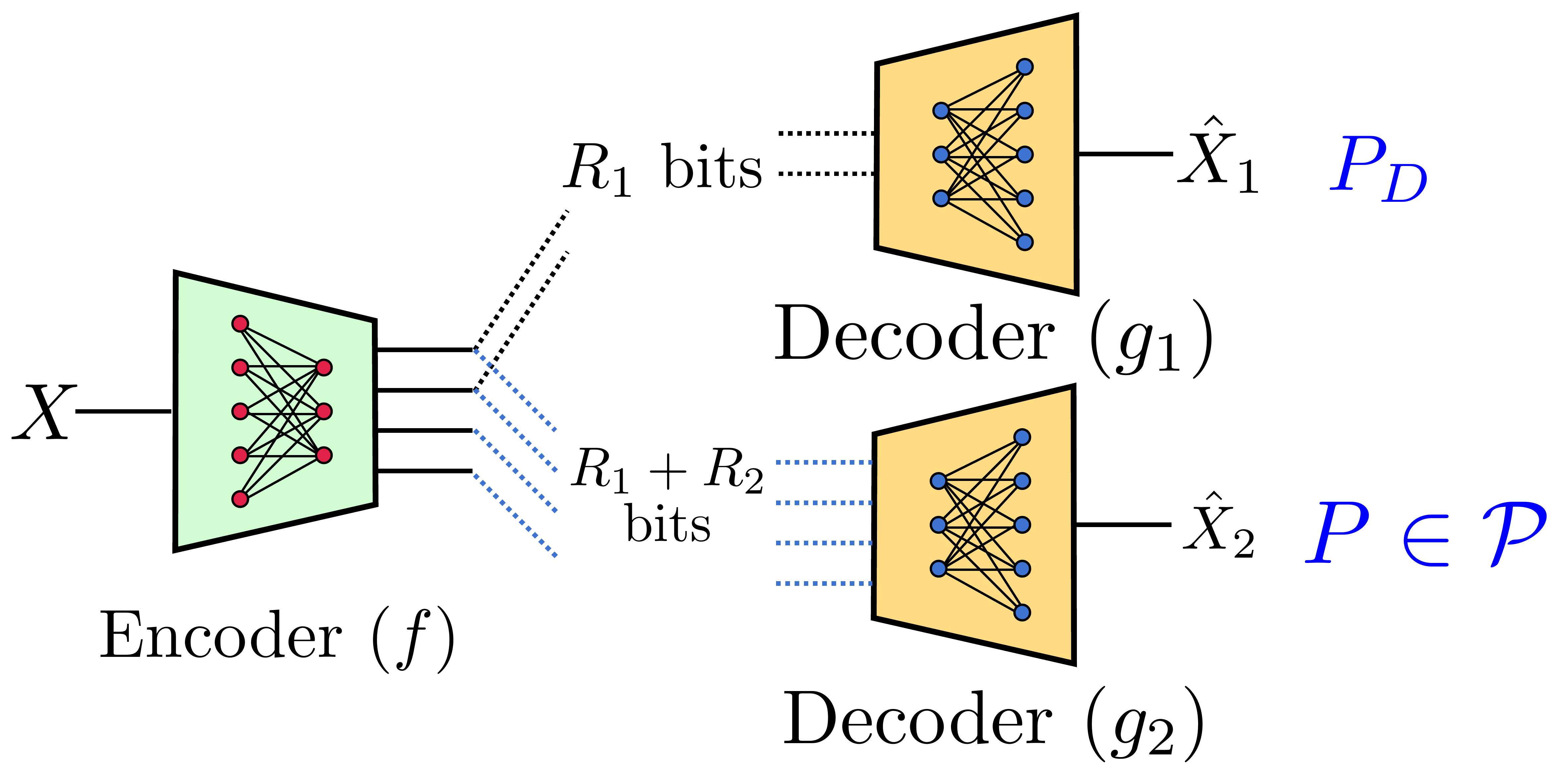}
    \caption{Structured compressor architecture with single encoder and multiple decoders.   }
    \label{fig:structured}

\end{figure}
In this model, we take the first $m_1$ entries of the output of $f_1$ to form $z_1$ (input to $g_1$) and the entire output ($m_1+m_2$ entries) of $f_1$ to form $[z_1, z_2]$ (input to $g_2$). To train, we sample a batch from $P_D$ and choose a decoder with probability $1/2$. If $g_1$ is chosen, we forward pass through the upper stage and minimize \eqref{eq:standard}. Otherwise, if $g_2$ is chosen, we forward pass through $g_2$ and minimize the robust objective \eqref{eq:lagrangian} by performing the procedure in Alg.~\ref{alg:DRO}.

\section{Experimental Results}
\begin{figure*}[!ht]
    \centering
    \includegraphics[width=1.0\linewidth]{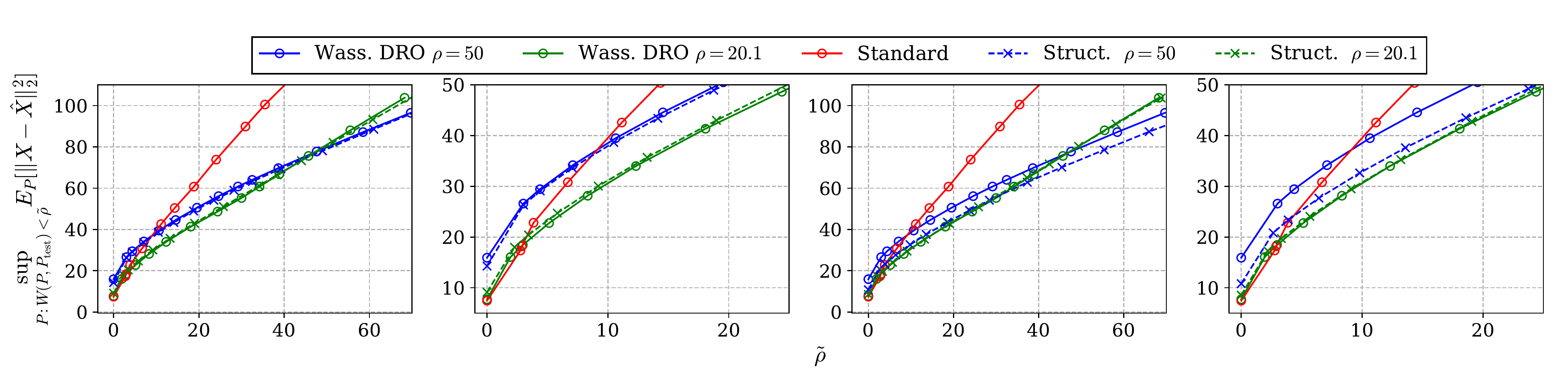}
    \vspace{-2.5em}
    \caption{Worst-case distortion comparisons for $R = 35.85$ bpi. $2^{\textrm{nd}}$ and $4^\textrm{th}$ plots are zoomed in verstions of the $1^{\textrm{st}}$ and $3^{\textrm{rd}}$. For structured codes, first two plots have rate split of 14-22 (14 bits for $1^{\textrm{st}}$ stage and 22 additional bits for $2^{\textrm{nd}}$ stage), last two plots have rate split 25-11.}
    \label{fig:WCD_short}
\end{figure*}


In our experiments, we use the MNIST dataset to represent the source $P_D$ at the center of the ball. We use the squared-error loss to represent the distortion $d$ and transport cost $c$ with $p=2$, i.e. $c(x, \hat{x}) = d(x,\hat{x}) = \|x-\hat{x}\|_2^2$. We analyzed structured compressors of various rate allocations, by varying the number of entries per stage while keeping their quantization levels fixed at $|\mathcal{W}|=12$. We train various models with different values of $\gamma$ to find those that enforce robustness to balls of radius $20.1$ and $50$. 

\subsection{End-to-end vs. Structured Models} In Fig.~\ref{fig:WCD_short}, we evaluate the worst-case test distortion \[\sup_{P:W_c(P, P_{\textrm{test}}) < \tilde{\rho}} \mathbb{E}_{P}[\|X-\hat{X}\|_2^2]\]
at various radii $\tilde{\rho}$ for standard, end-to-end DRO (denoted Wass.~DRO), and structured compressors. Note that $\tilde{\rho}$ is not the same as $\rho$, which is the radius the models were trained to be robust to. For structured, we evaluate the $2^{\textrm{nd}}$ stage since that compressor is designed to be robust. For larger values of $\tilde{\rho}$, the standard compressor achieves a much larger distortion than all the robust models, demonstrating that robustness to OOD sources has been achieved. However, for small values of $\tilde{\rho}$, the robust models are all strictly worse than the standard compressor, and when more robustness is desired, the performance on the center $P_D$ gets worse. In other words, to enforce robustness to a set of sources, one must trade off performance on the center. Also, compressors robust to larger radii do not perform as well as a compressors robust to smaller radii when evaluated on data with smaller perturbations.

Another interesting observation is that depending on the rate allocation and trained radius, the structured compressor either performs similarly to or better than Wass.~DRO. For the trained radius of $50$, allocating more bits to the first stage actually helps the second stage be more robust than Wass.~DRO across all test radii $\tilde{\rho}$ (compare blue curves). However, this is not observed for the trained radius of $\rho=20.1$, where the structured compressor's second stage performs similarly regardless of the rate allocation. One potential explanation for this phenomenon is that adding structure helps the model find more optimal saddle points for the optimization problem in (\ref{eq:robust}), which is non-convex in the neural network parameters.

\subsection{Comparison with Data Augmentation}
We compare our end-to-end and structured models with a simple data augmentation scheme: adding Gaussian noise to the dataset during training. Specifically, for training the robust models to radius $\rho$, we add Gaussian noise $\mathcal{N}(0, \sqrt{\frac{\rho}{n}}I_n)$, where $n$ is the dimension of the source. This ensures that the augmented distribution, $P_D * \mathcal{N}(0, \sqrt{\frac{\rho}{n}}I_n)$, lies on the boundary of the Wasserstein ball.

As shown in Fig.~\ref{fig:DA}, the worst-case distortion achieved by this scheme at $\tilde{\rho} = \rho$, while lower than the standard model, is higher than that achieved by the structured and end-to-end models. This implies that while simply augmenting the dataset with Gaussian noise can help robustness, $P_D * \mathcal{N}(0, \sqrt{\frac{\rho}{n}}I_n)$ does not achieve the worst-case distortion over the Wasserstein ball in general.

\begin{figure}[t]
    \centering
    \includegraphics[width=0.8\linewidth]{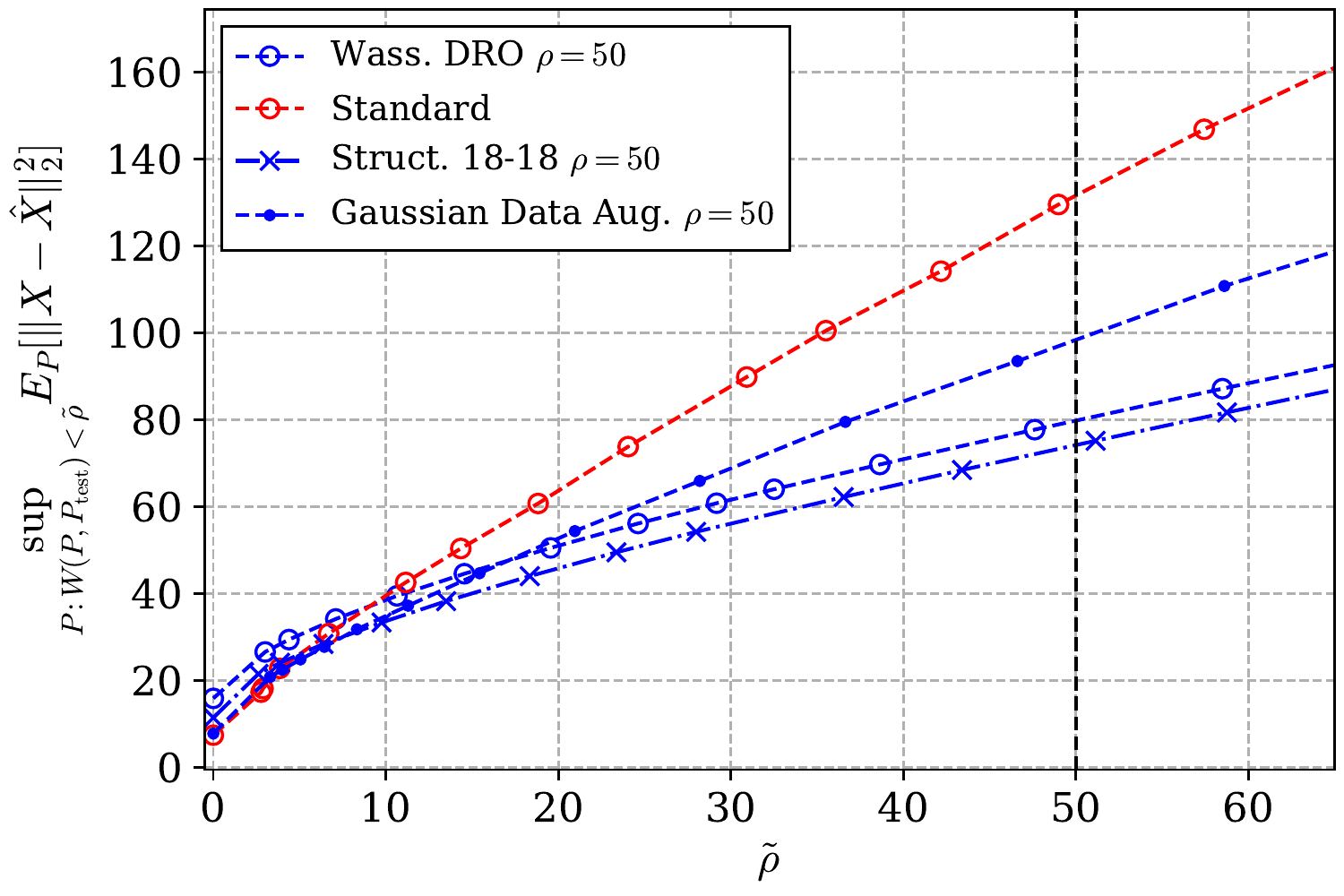}
    \caption{Comparison with AWGN data augmentation.}
    \label{fig:DA}
\end{figure}

\section{Group Shifts}
\label{sec:group}
Another choice of $\mathcal{P}$ contains distributions induced by group actions $c \in \mathcal{G}$ of $X \sim P_D$, i.e. $\mathcal{P} = \{P_{c \cdot X} : c \in \mathcal{G}\}$. In this paper, we specifically examine the case in which the group actions $\mathcal{G}$ contain all spatial rotations $c_\varphi$ of the images from angles $\varphi \in [-\pi/2, \pi/2)$.  The robust objective then becomes 
\begin{equation}
    \min_{\theta, \phi} \sup_{-\frac{\pi}{2} \leq \varphi < \frac{\pi}{2}} \EE_{ P_D}\bracket*{d(c_\varphi X, g_\phi(f_\theta(c_\varphi X)))} 
    \label{eq:group}
\end{equation}

Using a standard autoencoder architecture, we can solve the inner max directly using global optimization methods over $\varphi$, which is tractable since $\varphi$ is one-dimensional. 

The structured architecture, in this case, attempts to predict the group action $c$ with some estimate $c_{\textrm{pred}}$ and then undo the action, i.e. $c_{\textrm{pred}}^{-1} c X$. Then, $c_{\textrm{pred}}$ is compressed using $R_1$ bits, while $c_{\textrm{pred}}^{-1} c X$ is compressed using a standard compressor (trained on $P_D$) with $R_2$ bits. At the decoder, both are decompressed and the reconstruction is $\hat{X} = \hat{c}_{\textrm{pred}}\widehat{c_{\textrm{pred}}^{-1} c X}$. The spatial transformer architecture \cite{spatial_transformer} was used to estimate and undo the group actions, shown in Fig.~\ref{fig:structured_rotation}. In our case, we use $R_1 = 7.17$, and $R_2 = 28.68$ as our bit allocation. This roughly corresponds to quantizing the predicted group shift angle to every 2.5 degrees. 

\begin{figure}[t]
    \centering
    \includegraphics[width=0.5\textwidth]{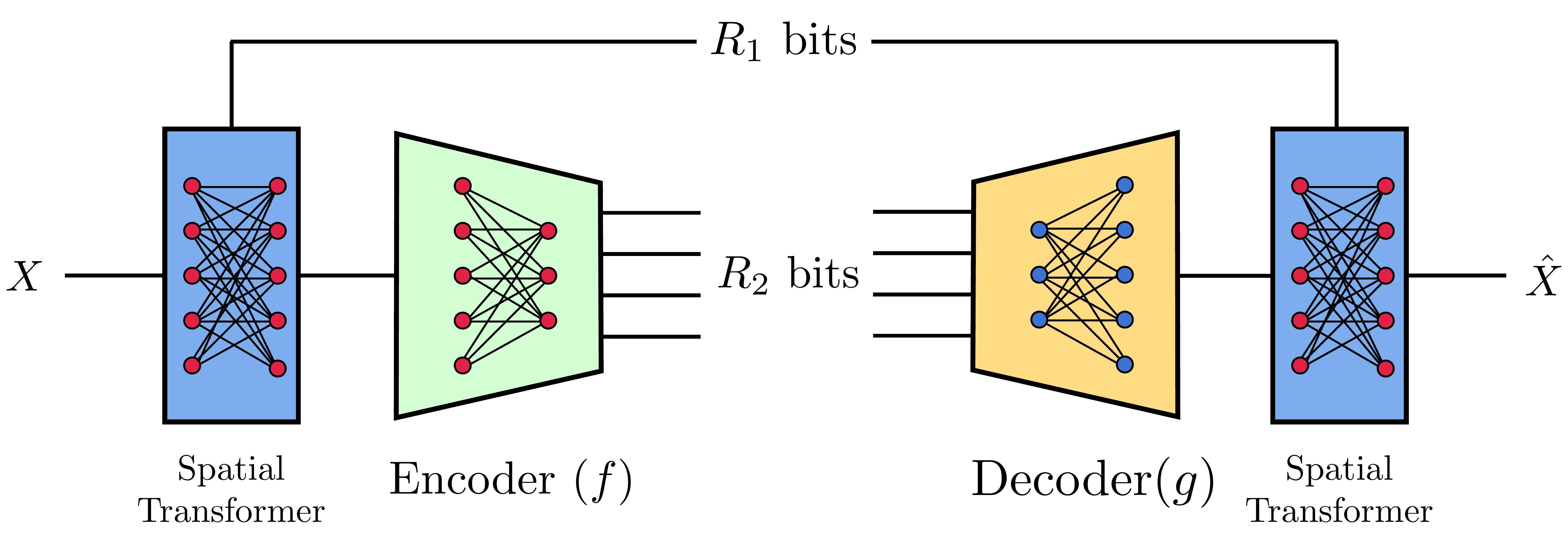}
    \caption{Structured architecture for group shifts: a portion ($R_1$) of the bits are used to encode the group shift.}
    \label{fig:structured_rotation}
\end{figure}

\begin{figure}[t]
    \centering
    \includegraphics[width=0.38\textwidth]{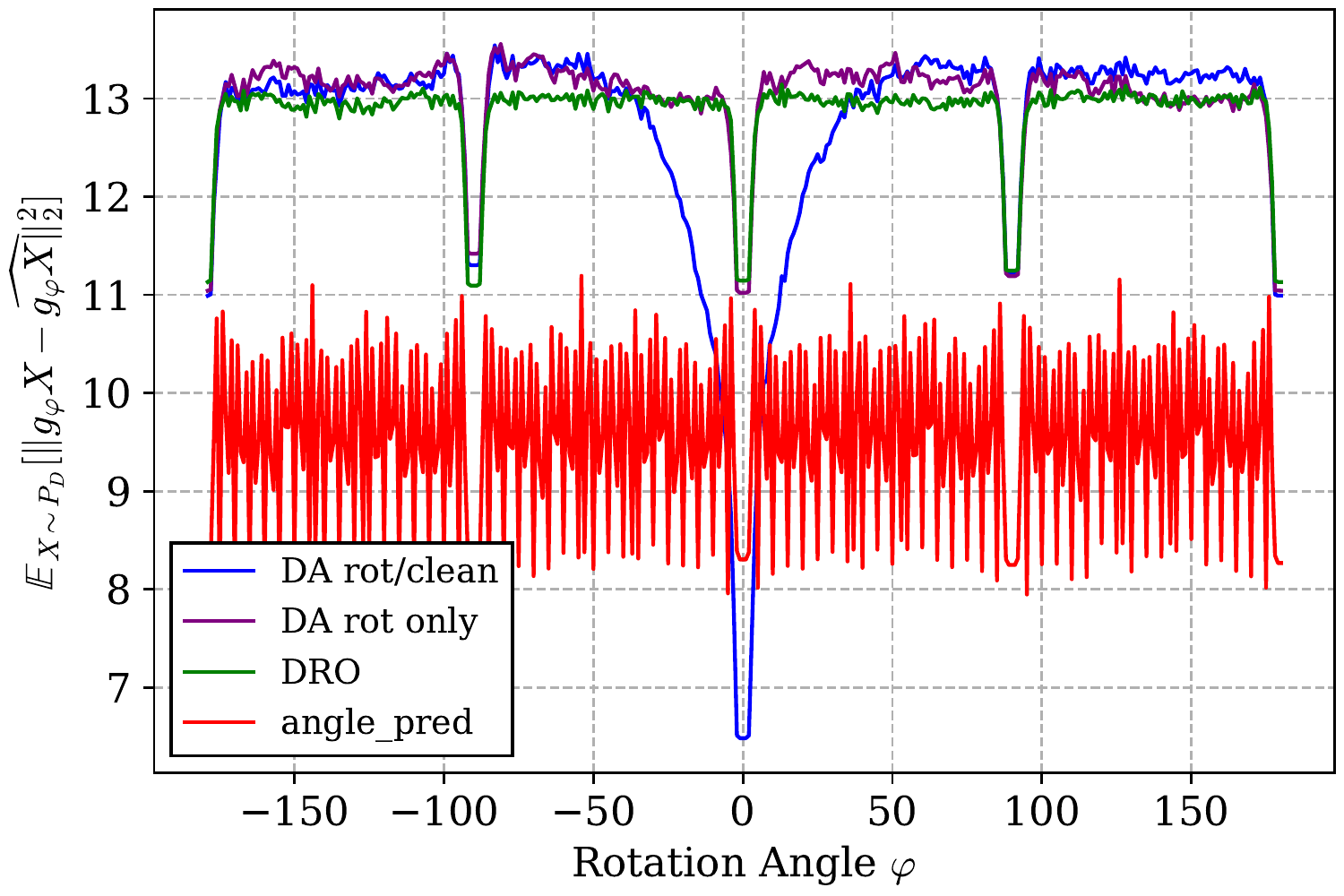}
    \caption{Average distortion over rotated images with angles $\varphi \in [-\pi/2, \pi/2)$. ``angle\_pred'' corresponds to the structured architecture for group shifts, where ``DRO'' corresponds to end-to-end training. Noisiness in ``angle\_pred'' is due to quantization of the predicted angle.}
    \label{fig:rotations}
\end{figure}

In Fig.~\ref{fig:rotations}, we show the distortion achieved over all angles $\varphi \in [-\pi/2, \pi/2)$, i.e. 
\begin{equation}
    \mathbb{E}_{P_D} [\norm*{c_\varphi X - g_\phi(f_\theta(c_\varphi X))}^2]
\end{equation} as a function of $\varphi$. We compare the end-to-end DRO and structured models against two models trained with data augmentation: one which duplicates each MNIST image with a randomly rotated one and concatenates it to unrotated MNIST, and one which only contains randomly rotated MNIST images. As can be seen, the end-to-end DRO performs roughly the same as the data augmentation with rotations only. However, using the structured architecture in Fig.~\ref{fig:structured_rotation} is able to achieved a much lower distortion uniformly over all $\varphi$. As with the Wasserstein ball, this implies that for robustness, using a structured code can indeed achieve superior performance. Also, this performance improvement cannot be attributed to the spatial transformer architecture. This is because we tried the end-to-end DRO with spatial transformers, but with no communication of the predicted angle, i.e. $R_1 = 0$, and this performed similarly to the model without spatial transformers. 

Interestingly, data augmentation with both rotated and unrotated images was able to achieve similar performance to the end-to-end DRO, but achieve very low distortion on unrotated images. 

\section{Theoretical Tradeoffs for Uniform Sources} \label{sec:tradeoff}
	In this section, we discuss tradeoffs in minimax-distortion single-shot compression of uniform sources. We show that in order to minimize worst-case distortion over an ambiguity set $\mathcal{P} = \{\Unif(0, 1), \Unif(0, 1+\delta)\}$, where $0 < \delta < \frac{1}{N}$, using a scalar quantizer, one must sacrifice performance on $\Unif(0, 1)$. For $N$-level scalar quantizers of $\Unif(0, \alpha)$, we only need to consider those given by interval partitions of $(0,\alpha)$  of lengths $\vec{p}(\alpha)=(p_1,\dots,p_N)$ and centers at their midpoints \cite{GyorgyLinder}. The optimal quantizer is $\vec{p}^*(\alpha)=(\frac{\alpha}{N}, \dots, \frac{\alpha}{N})$. We will denote these quantizers as $Q_{\vec{p}(\alpha)}$. 
	
	Define the expected squared-error distortion of a quantizer $Q$ on $\Unif(0, \alpha)$ as \[D(\alpha, Q):= \mathbb{E}_{X\sim \Unif(0,\alpha)} [(X-Q(X))^2]\] and define \[V(Q) := \max_{\alpha \in \{1, 1+\delta\}} D(\alpha, Q)\] as the worst-case distortion over $\mathcal{P}$. Broadly, the argument is that the individually-optimal quantizers of $\mathcal{P}$ are unique, while failing to be minimax optimal. 
	
	\begin{lemma}\emph{\cite{GyorgyLinder}}
	For all $N$-level scalar quantizers $Q$, 
	\begin{equation}
	    D(\alpha, Q) \geq \frac{\alpha^2}{12N^2}
	\end{equation} with equality if and only if $Q = Q_{\vec{p}^*(\alpha)}$. This holds more generally for all scalar quantizers $Q$ with entropy (of the output marginal distribution of $Q$) less than or equal to $\log N$.
	\end{lemma}
	
    \setcounter{lemma}{1}
    \begin{lemma}\label{lemma:WCD}
    The worst-case distortion using the optimal quantizer for $\Unif(0,1+\delta)$ is less than the worst-case distortion using the optimal quantizer for $\Unif(0,1)$, i.e.
        \begin{equation}
        V(Q_{\vec{p}^*({1+\delta})}) < V(Q_{\vec{p}^*(1)}).
        \end{equation}
    \end{lemma}
    
    \begin{proof}
        See appendix A.
    \end{proof}
    
    \begin{definition}[Minimax quantizer]
	The minimax $N$-level quantizer is defined as 
	\begin{equation}
	    Q^* = \argmin_{Q: |Q|=N} V(Q).
	\end{equation} 
	\end{definition}

	\begin{thm}[Trade-off] \label{claim:tradeoff}
	The minimax $N$-level quantizer achieves a strictly greater distortion when compressing $\Unif(0,1)$ than the $N$-level quantizer optimal for $\Unif(0,1)$, i.e. 
	\begin{equation}
	\vspace{-.2cm}
	    D(1, Q^*) > D(1, Q_{\vec{p}^*(1)}).
	\end{equation}
	\end{thm}
	\begin{proof} \label{proof:tradeoff}
	Suppose the opposite is true, i.e. $D(1, Q^*) = D(1, Q_{\vec{p}^*(1)})$. By Lemma 1, this implies that $Q^* =  Q_{\vec{p}^*(1)}$. By Lemma 2, we know there exists a quantizer that achieves lower worst-case distortion than $Q_{\vec{p}^*(1)}$, in particular, $Q_{\vec{p}^*({1+\delta})}$. Hence $Q^*$ cannot be the minimax quantizer and we have a contradiction. 
	\end{proof}
 Hence, one must  trade off minimizing worst-case distortion of multiple uniform sources with minimizing distortions of individual sources. 

\vspace{-.2cm}
\section{Conclusion}
\vspace{-.1cm}
In this paper, we introduced and studied the problem of OOD-robust compression. We considered robustness to distributions pertaining to two ambiguity sets: those within a Wasserstein ball of a certain radius around the data distribution, and those induced by group shifts of the data. We proposed two methods based on (1) distributionally-robust optimization and (2) structured codes.  We showed experimentally that both can help with robustness to OOD samples, and that structured codes can outperform  minimizing a robust objective. Theoretical tradeoffs for achieving robustness support tradeoffs observed empirically. There are many further avenues for investigation for OOD-robust compression beyond the Wasserstein ball and group shifts. Moreover, the connection between rate, robustness, and code structure is not clear, much less any theoretical understanding. 

\section*{Acknowledgements}
This material is based upon work supported by a NSF Graduate Research Fellowship and NSF grant CCF-1910056. 

\bibliography{ref}
\bibliographystyle{icml2021}

\onecolumn

\appendix

\section{Proof of Lemma 2.} 
\setcounter{lemma}{1}
\begin{lemma}
    The worst-case distortion using the optimal quantizer for $\Unif(0,1+\delta)$ is less than the worst-case distortion using the optimal quantizer for $\Unif(0,1)$, i.e.
        \begin{equation}
        V(Q_{\vec{p}^*({1+\delta})}) < V(Q_{\vec{p}^*(1)}).
        \end{equation}
    \end{lemma}
\begin{proof} \label{proof:WCD}
    First, observe that on the RHS, $V(Q_{\vec{p}^*(1)}) = D(1+\delta, Q_{\vec{p}^*(1)})$. This is because 
    \begin{align*}
        D(1, \Qopt{1}) &= \frac{1}{12N^2} \\
        &\leq \frac{(1+\delta)^2}{12N^2} \\
        &= D(1+\delta, \Qopt{1+\delta}) \\
        &< D(1+\delta, \Qopt{1})
    \end{align*}
     where the last step follows from Lemma 1. 
    
    The last step above along with the fact that $V(Q_{\vec{p}^*(1)}) = D(1+\delta, Q_{\vec{p}^*(1)})$ implies that it suffices to show that $D(1, \Qopt{1+\delta}) < D(1+\delta, \Qopt{1})$. 
    
    To show this, observe that for $N$-level quantizers $Q_\vec{p}$ defined by a sequence of disjoint but adjacent intervals of lengths $p_1,\dots,p_N$ with centers at the midpoints of the intervals, $D(\alpha, Q_{\vec{p}})$ is as follows:
	
	\begin{equation*}
	    D(\alpha, Q_{\vec{p}}) = 
	    \begin{dcases}
	        \frac{1}{3\alpha}\bracket*{\paran*{\alpha-\frac{1}{2}p_N-\sum_{i=1}^Np_i}^3 - \paran*{\frac{1}{2}p_N}^3} + \frac{1}{12\alpha}\sum_{i=1}^N p_i^3 & \text{if }\alpha \geq \sum_{i=1}^N p_i \vspace{1em} \\ 
	        \frac{1}{3\alpha}\bracket*{\paran*{\alpha-\frac{1}{2}p_K-\sum_{i=1}^Kp_i}^3 + \paran*{\frac{1}{2}p_K}^3} + \frac{1}{12\alpha}\sum_{i=1}^{K-1} p_i^3 & \text{if } \sum_{i=1}^{K-1}p_i \leq \alpha < \sum_{i=1}^K p_i, 2 \leq K \leq N
	    \end{dcases}
	\end{equation*}
	This is essentially the distortion for all the intervals that are completely contained in the support of $\Unif(0,\alpha)$, plus the distortion for the remaining mismatched part. 
	
	$D(1+\delta, \Qopt{1})$ falls into the first case, since $\alpha=1+\delta$ and each $p_i = \frac{1}{N}$ for $\vec{p}^*(1)$. As for $D(1, \Qopt{1+\delta})$, this falls into the second case with $K=N$ because we assumed $\delta < \frac{1}{N}$. Using these formulas,
	\[D(1, \Qopt{1+\delta}) = \frac{1}{24N^3}f_1(N,\delta)\]
	\[D(1+\delta, \Qopt{1}) = \frac{1}{24N^3}f_2(N,\delta)\]
	where 
	\[f_1(N,\delta) = \paran*{2N-(1+\delta)(2N-1)}^3+(2N-1)(1+\delta)^3 \]  
	\[f_2(N,\delta) = \paran*{(1+\delta)^{2/3}2N - (1+\delta)^{-1/3}(2N-1)}^3+(2N-1)(1+\delta)^{-1}\]
	
	The derivatives with respect to $\delta$ are 
	\[\frac{\partial f_1}{\partial \delta} = -3(2N-(1+\delta)(2N-1))^2(2N-1)+3(2N-1)(1+\delta)^2\]
	\[\frac{\partial f_2}{\partial \delta} = \paran*{(1+\delta)^{2/3}2N-(1+\delta)^{-1/3}(2N-1)}^2(4N(1+\delta)^{-1/3}+(1+\delta)^{-4/3}(2N-1)) - (2N-1)(1+\delta)^{-2}\]
	It can be verified that for $0 \leq \tilde{\delta} \leq \frac{1}{N}$, $\frac{\partial f_2}{\partial \delta} \Bigr|_{\delta=\tilde{\delta}} \geq \frac{\partial f_1}{\partial \delta} \Bigr|_{\delta=\tilde{\delta}} \geq 0$.
	Furthermore, $f_1(N,0) = f_2(N,0)$, and hence $f_1(N,\tilde{\delta}) \leq f_2(N,\tilde{\delta})$ for $0 \leq \tilde{\delta} \leq \frac{1}{N}$, which concludes the proof.

    \end{proof}
    


\end{document}